\pdfoutput=1
\documentclass[11pt,a4paper]{article}
\usepackage[margin=25mm]{geometry} 
\usepackage{setspace}
\usepackage{graphicx}
\usepackage{mathtools,amssymb,amsthm,bm}
\allowdisplaybreaks
\usepackage{algorithm}
\usepackage[noEnd=true]{algpseudocodex}
\usepackage{subcaption}
\usepackage{cite}
\usepackage{tikz}
\usetikzlibrary{arrows.meta,positioning,decorations.pathreplacing,calc,shapes.multipart}
\usepackage{enumitem}
\setlist[itemize]{leftmargin=*}
\usepackage{xspace}
\usepackage{comment}
\newcommand{\ltwolone}{\ensuremath{\ell_2/\ell_1}\xspace}

\DeclareMathOperator*{\minimize}{minimize}
\usepackage[capitalize,nameinlink]{cleveref}
\crefname{figure}{Fig.}{Figs.}
\crefname{table}{Tab.}{Tabs.}
\crefname{equation}{}{}
\crefname{algorithm}{Alg.}{Algs.}

\usepackage{authblk}
\usepackage{wrapfig} 

\newcommand{\figconceptual}{} 

\newcommand{\tblarchcomp}{
\begin{table}[t]
    \centering
    \caption{Comparison between conventional LOP-\ltwolone and DU-LOP-\ltwolone methods.}
    \label{tbl:arch_comp}
    \begin{tabular}{lccc} \hline
        Feature            & LOP-\ltwolone~\cite{9729560} & \textbf{PS-DU-LOP-\ltwolone} & \textbf{P-GD-DU-LOP-\ltwolone} \\
        \hline
        Optimization       & Proximal Splitting           & Proximal Splitting           & Preconditioned GD              \\
        Parameter Tuning   & Manual                       & \textbf{Automatic}           & \textbf{Automatic}             \\
        Available Loss     & Prox-friendly                & Prox-friendly                & \textbf{Differentiable}        \\
        Backprop Stability & Unstable                     & \textbf{Stable}              & \textbf{Stable}                \\
        \hline
    \end{tabular}
\end{table}
}

\newcommand{\figcompflow}{
\begin{wrapfigure}{r}{0.4\textwidth}
    \centering
    \begin{tikzpicture}[
            node distance=0.5cm and 0.6cm,
            block/.style={draw, rectangle, minimum height=1.2em, minimum width=3.0em, rounded corners, align=center, font=\small},
            arrow/.style={-Stealth, thick}
        ]
        \node (in) [block] {$(x, \sigma)$};
        \node (pq) [block, right=of in] {$(p, q)$};
        \node (s) [block, right=of pq] {$s$};
        \node (out) [block, below=0.2cm of pq] {$\text{prox}_{\gamma \lambda \phi}(x, \sigma)$};

        \draw [arrow] (in) -- node[above, scale=1] {} (pq);
        \draw [arrow] (pq) -- node[above, scale=1] {\cref{eq:cubic}} (s);
        \draw [arrow] (in.south) |- (out.west);
        \draw [arrow] (s.south) |- (out.east);
    \end{tikzpicture}
    \caption{Computation flow of $\text{prox}_{\phi}$.}
    \label{fig:comp_flow}
\end{wrapfigure}
}

\newcommand{\figstability}{
\begin{figure}[t]
    \centering
    \includegraphics[width=0.6\linewidth]{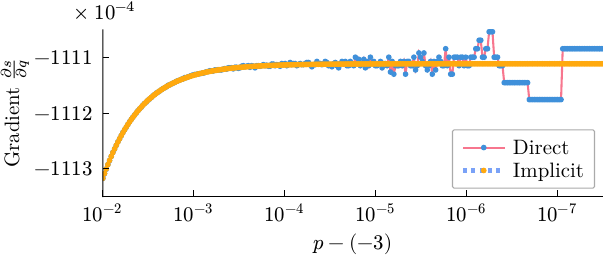}
    \caption{Numerical stability comparison of gradients $\partial s / \partial q$ near the multiple root boundary ($D \nearrow 0, q=-2$).}
    \label{fig:stability_plot}
\end{figure}
}

\newcommand{\figphasetransition}{
\begin{figure}[t]
    \centering
    \begin{subfigure}[b]{0.9\linewidth}
        \centering
        \includegraphics[width=1.0\linewidth]{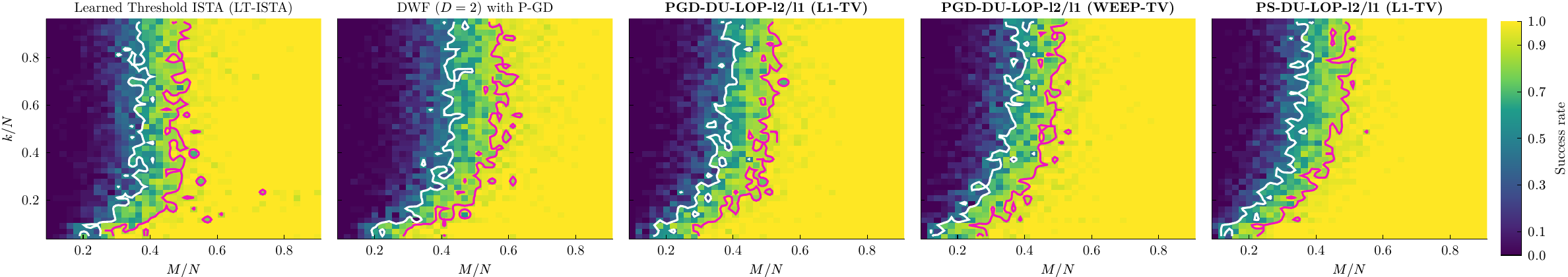}
        \caption{Phase transition comparison. White and magenta contour lines denote 50\% and 90\% success rates, respectively.}
        \label{fig:pt_block}
    \end{subfigure}
    \begin{subfigure}[b]{0.75\linewidth}
        \centering
        \includegraphics[width=1.0\linewidth]{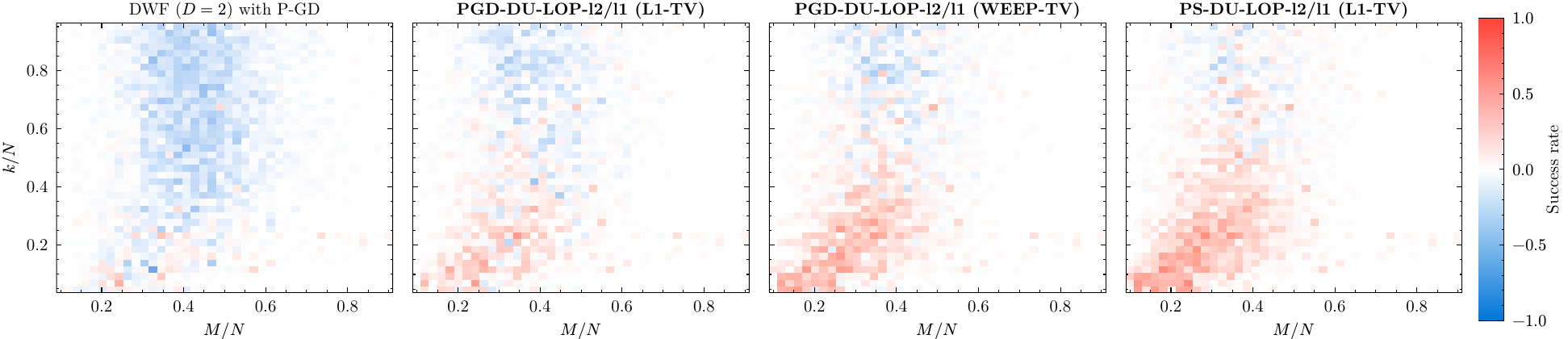}
        \caption{Differences between phase transition diagrams. Red regions indicate higher success than LT-ISTA.}
        \label{fig:pt_diff}
    \end{subfigure}
    \caption{Phase transition analysis across the measurement rate $M/N$ and sparsity rate $k/N$ grid.}
    \label{fig:phase_transition}
\end{figure}
}

\newcommand{\figptoutlier}{
\begin{figure}[t]
    \centering
    \includegraphics[width=0.6\linewidth]{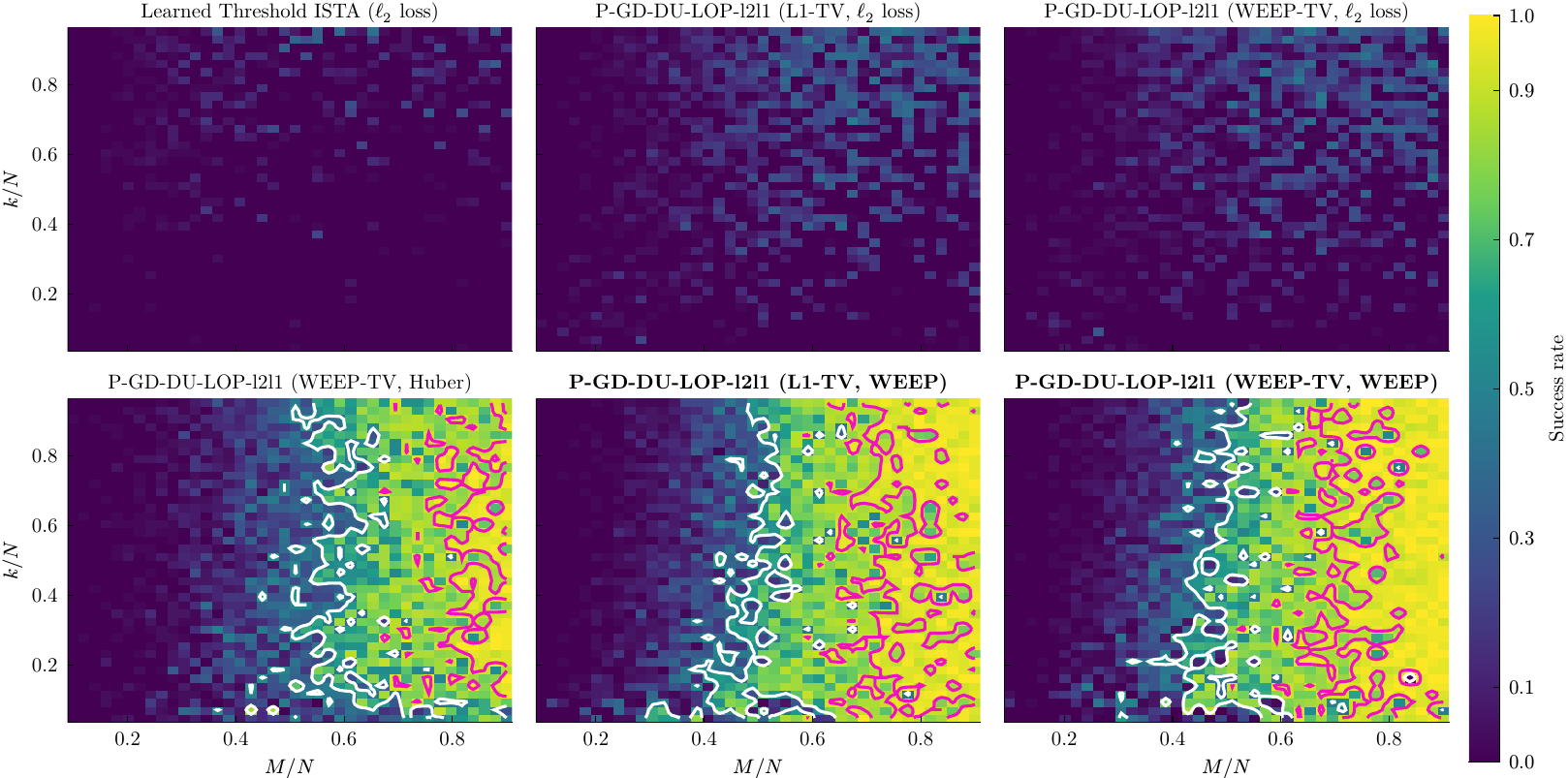}
    \caption{Phase transition under impulsive noise. The P-GD with the WEEP fidelity shows superior resilience to outliers.}
    \label{fig:pt_outlier}
\end{figure}
}

\begin{document}

\title{\textbf{Deep Unfolded Latent Optimally Partitioned-\ltwolone Networks for Data-driven Block-Sparse Recovery}}

\author[1,2]{Takanobu Furuhashi}
\author[1]{Hidekata Hontani}
\author[2]{Qibin Zhao}
\author[1,2]{Tatsuya Yokota}

\affil[1]{Nagoya Institute of Technology}
\affil[2]{RIKEN Center for Advanced Intelligence Project}

\date{}
\maketitle

\section*{Abstract}
The convex Latent Optimal Partition (LOP)-\ltwolone approach enables block-sparse signal recovery with unknown partitions but relies on manual hyperparameter tuning.
Additionally, numerical instability in differentiating its proximal operator prevents its automatic parameter tuning via Deep Unfolding (DU).
To address these limitations, we propose two architectures: a stable framework utilizing implicit differentiation and a flexible variant leveraging Deep Weight Factorization (DWF).
The DWF-based approach also supports nonconvex smooth data fidelity terms.
Numerical experiments demonstrate that DU-LOP-\ltwolone yields competitive performance and high resilience against impulsive noise.

\section{Introduction}

Block-sparse signal recovery is fundamental to applications such as image data mining and radar imaging~\cite{qianHyperspectralImageryRestoration2013,wangEnhancedISARImaging2014,gaoBlockSparseRPCASalient2014,liuBackgroundSubtractionBased2015,liStructuredSparseCoding2022,zhaoHyperspectralAnomalyDetection2023,kurodaTheoreticalValidationLatent2025}. This paradigm typically addresses ground truth whose nonzero entries cluster in contiguous blocks.

We address signal recovery from the linear observation model $\bm{y} = \bm{A}\bm{x} + \bm{\epsilon}$,
where $\bm{x} \in \mathbb{R}^N$ is the block-sparse signal, $\bm{y} \in \mathbb{R}^M$ represents the measurements, and $\bm{A} \in \mathbb{R}^{M \times N}$ denotes the sensing matrix.
$\bm{\epsilon} \in \mathbb{R}^M$ is additive noise, which often contains outliers in practical scenarios~\cite{carrilloRobustSamplingReconstruction2010,studerRecoverySparselyCorrupted2012,liuBackgroundSubtractionBased2015,yuanL0TVNewMethod2015,wanRobustBayesianCompressed2017}.

A critical challenge is that underlying block partitions are seldom known \textit{a priori}.
While conventional mixed \ltwolone-norm (Group Lasso)~\cite{yuanModelSelectionEstimation2006} penalties are effective for fixed partitions~\cite{stojnicReconstructionBlockSparseSignals2009,eldarBlocksparseSignalsUncertainty2010,lvGroupLassoStable2011,zhangExtensionSBLAlgorithms2013}, they suffer severe performance degradation when the assumed partitions do not match the actual ones~\cite{9729560,zhangExtensionSBLAlgorithms2013}. To resolve this issue, Kuroda and Kitahara~\cite{9729560} introduced the Latent Optimally Partitioned (LOP)-\ltwolone penalty. This method discovers the optimal partition and calculates the value of the mixed \ltwolone-norm via iterative convex optimization, which is computationally more efficient than existing greedy and Bayesian approaches~\cite{baraniukModelBasedCompressiveSensing2010,zhangExtensionSBLAlgorithms2013,fangPatterncoupledSparseBayesian2015,santBlockSparseSignalRecovery2022}.

A practical limitation of LOP-\ltwolone lies in the manual hyperparameter tuning.
Specifically, it requires appropriately setting both the regularization strength for sparsity ($\lambda$) and the total variation bound controlling the number of block partitions ($\alpha$) as partly discussed in \cite{9729560}.
Exploring their combinations via grid search is often impractical due to the vast search space.
Deep Unfolding (DU)~\cite{hersheyDeepUnfoldingModelBased2014,balatsoukas-stimmingDeepUnfoldingCommunications2019,mongaAlgorithmUnrollingInterpretable2021,shlezingerDeepUnfoldingRecent2025} resolves this by unrolling the iterative optimization steps into layers of a neural network. By treating the hyperparameters as learnable weights, DU leverages backpropagation on a training dataset to automatically optimize them, replacing manual search with data-driven estimation.

However, most existing DU methods for sparse recovery~\cite{borgerdingAMPInspiredDeepNetworks2017,zhangISTANetInterpretableOptimizationInspired2018,itoTrainableISTASparse2019,fuDeepUnfoldingNetwork2021,gaoDeepLearningBasedChannel2023,wuRPCANetDeepUnfolding2024}, including the learned iterative shrinkage thresholding algorithm (LISTA)~\cite{gregorLearningFastApproximations2010}, cannot exploit block-sparse priors or require predefined block boundaries, which prevents them from handling unknown partitions.
Furthermore, proximal splitting solvers~\cite{9729560,kurodaConvexnonconvexFrameworkEnhancing2025} for the LOP-\ltwolone method rely on Cardano's formula to solve cubic equations~\cite{bauschkeRealRootsReal2023}, which leads to numerical instabilities during backpropagation, hindering direct deep unfolding.
They also restrict the data fidelity to convex and \textit{prox-friendly} functions, limiting the use of robust nonconvex losses against outliers.

\begin{figure*}[t]
    \begin{minipage}[b]{0.49\textwidth}
        \centering
        \resizebox{\linewidth}{!}{\begin{tikzpicture}[
        node distance=0.8cm,
        block/.style={rectangle, draw, thick, text width=3.8cm, align=center, rounded corners, minimum height=0.7cm, font=\footnotesize},
        header/.style={
                rectangle split, rectangle split parts=2,
                draw, thick, text width=3.8cm, align=center, rounded corners,
                minimum height=0.7cm, font=\footnotesize,
                rectangle split part fill={cyan!20, cyan!10}
            },
        arrow/.style={-Stealth, thick}
    ]
    \node (lop) [block] {\textbf{LOP-\ltwolone} with \\ \scriptsize Proximal Splitting  (\cref{alg:kuroda})};
    \node (du) [block, right=0.3cm of lop] {\textbf{Deep Unrolling (DU)} \\ \scriptsize Data-driven Learning};

    \coordinate (merge_point) at ($(lop.south)!0.5!(du.south) + (0,-0.4cm)$);

    \node (ours_implicit) [header, below=1.2cm of lop] {
        \textbf{PS-DU-LOP-\ltwolone}
        \nodepart{two}
        \scriptsize \textbf{Implicit Differentiation} \\
        \scriptsize for $\text{prox}_{\gamma \lambda \phi}(x, \sigma)$ in \cref{alg:kuroda}
    };
    \node (ours_dwf) [header, below=1.2cm of du] {
        \textbf{P-GD-DU-LOP-\ltwolone}
        \nodepart{two}
        \scriptsize \textbf{Differentiable LOP-\ltwolone} \\
        \scriptsize with DWF + P-GD (\cref{alg:pgd})
    };

    \draw [thick] (lop.south) -- (lop.south |- merge_point) -- (merge_point);
    \draw [thick] (du.south) -- (du.south |- merge_point) -- (merge_point);

    \draw [thick] (merge_point) -- ++(0,-0.3) coordinate (split);
    \draw [arrow] (split) -| (ours_implicit.north);
    \draw [arrow] (split) -| (ours_dwf.north);

    \node [font=\tiny, align=right, anchor=north east, inner sep=2pt] at (lop.south) {Block-Sparse \\ Signal Recovery};
    \node [font=\tiny, align=left, anchor=north west, inner sep=2pt] at (du.south) {Hyperparameter \\ Tuning};
    \node [font=\tiny, text=blue, align=center, anchor=north, inner sep=7pt] at (split) {↓Solves Backprop Instability↓};
    \node [font=\tiny, text=blue, align=left, anchor=south west, inner sep=3pt] at (ours_dwf.north) {Broadens Fidelity \\ Applicability↓};
\end{tikzpicture}}
        \caption{Relationship between the convex LOP-\ltwolone approach and our deep unfolded networks.}
        \label{fig:flowchart}
    \end{minipage}
    \hfill
    \begin{minipage}[b]{0.49\textwidth}
        \centering
        \resizebox{\linewidth}{!}{\begin{tikzpicture}[>=Stealth, thick, scale=0.6]
    \def\SignalValues{0,0,0,0, 1.2,-0.8,1.8,-1.0,-0.5, 0,0,0,0,0, 1.5, -0.7, 0,0,0, -0.8,1.7,1.1}
    \def\vscale{0.5} 

    \xdef\BlockDataList{}
    \def\cstart{1} \def\csqsum{0} \def\ccount{0} \def\ltype{-1}
    \foreach \v [count=\i] in \SignalValues {
        \pgfmathsetmacro{\vsq}{\v*\v}
        \pgfmathtruncatemacro{\ctype}{(\v == 0 ? 0 : 1)}
        \ifnum\i=1 \xdef\ltype{\ctype} \fi
        \ifnum\ctype=\ltype
            \pgfmathsetmacro{\nsqsum}{\csqsum + \vsq}
            \pgfmathsetmacro{\ncount}{\ccount + 1}
            \xdef\csqsum{\nsqsum} \xdef\ccount{\ncount}
        \else
            \pgfmathsetmacro{\btau}{sqrt(\csqsum / \ccount)}
            \xdef\BlockDataList{\BlockDataList \cstart/\the\numexpr\i-1\relax/\ltype/\btau, }
            \xdef\cstart{\i} \xdef\ltype{\ctype} \xdef\csqsum{\vsq} \xdef\ccount{1}
        \fi
        \xdef\finalidx{\i}
    }
    \pgfmathsetmacro{\btau}{sqrt(\csqsum / \ccount)}
    \xdef\BlockDataList{\BlockDataList \cstart/\finalidx/\ltype/\btau}

    \def\margin{0.5} 
    \def\dOne{2.3}  \def\hTwo{2*\vscale}  \def\dTwo{0.0}  \def\hThree{1.5*\vscale}
    \pgfmathsetmacro{\yOne}{0}
    \pgfmathsetmacro{\yTwo}{-(\dOne + \hTwo + \margin)}
    \pgfmathsetmacro{\yThree}{\yTwo - (\dTwo + \hThree + \margin)}

    \foreach \tier/\y [count=\ty from 0] in {x/\yOne, \sigma/\yTwo, D\sigma/\yThree} {
            \begin{scope}[yshift=\y cm]
                \draw[-] (-0.5, 0) -- (11, 0) node[right] {$\bm\tier$};

                \ifnum\ty=0 
                    \foreach \v [count=\i] in \SignalValues {
                        \pgfmathsetmacro{\x}{0.45*\i}
                        \ifdim\v pt=0pt \fill (\x, 0) circle (2.8pt);
                        \else \draw (\x, 0) -- (\x, \v*\vscale); \fill (\x, \v*\vscale) circle (2.8pt); \fi
                        \ifnum\i=1 \node[below, shift={(0,-0.1)}, font=\small] at (\x, 0) {$x_1$};\fi
                        \ifnum\i=2 \node[below, shift={(0.05,-0.1)}, font=\small] at (\x, 0) {$x_2$};\fi
                        \ifnum\i=3 \node[below, shift={(0.2,-0.15)}, font=\small] at (\x, 0) {$\dots$};\fi
                        \ifnum\i=25 \node[below, shift={(0,-0.1)}, font=\small] at (\x, 0) {$x_N$};\fi
                    }
                    \foreach \s/\e/\tyl/\btau [count=\bidx] in \BlockDataList {
                        \pgfmathsetmacro{\xs}{0.45*\s - 0.15} \pgfmathsetmacro{\xe}{0.45*\e + 0.15}
                        \begin{scope}[red, decoration={brace, mirror, amplitude=6pt}]
                            \draw[decorate] (\xs, -1.0) -- (\xe, -1.0) node[midway, below=8pt, font=\small] {$\mathcal{B}_{\bidx}$};
                        \end{scope}
                    }
                \fi

                \ifnum\ty=1 
                    \foreach \s/\e/\tyl/\btau [count=\bidx] in \BlockDataList {
                        \foreach \j in {\s,...,\e} {
                                \pgfmathsetmacro{\x}{0.45*\j}
                                \ifdim\btau pt=0pt \fill (\x, 0) circle (2.8pt);
                                \else \draw (\x, 0) -- (\x, \btau*\vscale); \fill (\x, \btau*\vscale) circle (2.8pt); \fi
                            }
                        \pgfmathsetmacro{\xs}{0.45*\s - 0.15} \pgfmathsetmacro{\xe}{0.45*\e + 0.15}
                        \draw[cyan, thick] (\xs, -0.3) rectangle (\xe, \btau*\vscale + 0.3);
                        \ifdim\btau pt=0pt \else \node[cyan, font=\footnotesize, anchor=south] at (0.5*\xs+0.5*\xe, \btau*\vscale + 0.3) {$\sigma_n = \tau_{\bidx}$}; \fi
                    }
                    \node[below, font=\small] at (0.45*1, -0.2) {$\sigma_1$};
                    \node[below, font=\small] at (0.45*2.2, -0.2) {$\sigma_2$};
                    \node[below, font=\small] at (0.45*3.8, -0.3) {$\dots$};
                    \node[below, font=\small] at (0.45*22, -0.2) {$\sigma_N$};
                \fi

                \ifnum\ty=2 
                    \foreach \s/\e/\tyl/\btau [remember=\btau as \lastbtau (initially 0)] in \BlockDataList {
                        \pgfmathsetmacro{\diff}{\btau - \lastbtau} \pgfmathsetmacro{\x}{0.45*(\s - 0.5)}
                        \ifdim\diff pt=0pt \ifnum\s=1 \else \fill (\x, 0) circle (2.8pt); \fi
                        \else \draw (\x, 0) -- (\x, \diff*\vscale); \fill (\x, \diff*\vscale) circle (2.8pt); \fi
                        \ifnum\s=\e \else \pgfmathtruncatemacro{\snext}{\s+1} \foreach \j in {\snext,...,\e} { \fill ({0.45*(\j-0.5)},0) circle (2.8pt); } \fi
                    }
                \fi
            \end{scope}
        }
\end{tikzpicture}}
        \caption{Conceptual illustration of block-sparse signal $\bm{x}$ and its partitioning scheme with a piecewise constant vector $\bm{\sigma}$.}
        \label{fig:unknown_partitions}
    \end{minipage}
\end{figure*}

Our main contributions are:
(i) \textbf{Stable Unfolded Networks} establishing both a proximal splitting-based network with implicit differentiation and a flexible gradient-based alternative using the reparameterization method called Deep Weight Factorization (DWF);
(ii) \textbf{Automatic Parameter Tuning} via a DU network that eliminates manual heuristic selection;
(iii) \textbf{Robust Block-Sparse Recovery} against outliers by making the LOP-\ltwolone penalty differentiable, facilitating the integration of robust and nonconvex penalties.

\textbf{Notation}: We denote scalars by lowercase letters ($x$), vectors by bold lowercase ($\bm{x}$), and matrices by bold uppercase ($\bm{A}$).
The $i$-th entry of $\bm{x}$ is $x_i$, and $\mathbb{R}$ is the set of real numbers.
$|\mathcal{B}|$ denotes the cardinality of a set $\mathcal{B}$.
$\|\bm{x}\|_p$ is the $\ell_p$-norm, and the mixed \ltwolone-norm is $\|\bm{x}\|_{2,1}^{(\mathcal{B}_k)_{k=1}^j} = \sum_{k=1}^j \sqrt{|\mathcal{B}_k|} \|\bm{x}_{\mathcal{B}_k}\|_2$ for block partitions $(\mathcal{B}_k)_{k=1}^j$.
$\odot$ denotes the element-wise (Hadamard) product, $\nabla_{\bm{x}}$ is the gradient with respect to $\bm{x}$, and $\bm{D} \in \mathbb{R}^{(N-1) \times N}$ is the first-order finite difference matrix defining $(\bm{Dx})_i = x_{i+1} - x_i$.
The proximal operator of a convex function $f: \mathbb{R}^N \to \mathbb{R} \cup \{\infty\}$ is $\text{prox}_f(\bm{v}) = \arg\min_{\bm{x} \in \mathbb{R}^N} \{ f(\bm{x}) + \frac{1}{2}\|\bm{x}-\bm{v}\|_2^2 \}$.

\section{Latent Optimally Partitioned \ltwolone}
\label{sec:lop-l2l1}
LOP-\ltwolone~\cite{9729560} minimizes the mixed \ltwolone penalty over all possible contiguous partitions, performing joint automatic block partitioning and block-sparse signal recovery. For a vector $\bm{x} \in \mathbb{R}^N$, the penalty is formulated as follows:
\begin{equation}
    \psi_K (\bm{x}) := \min_{j \in \{1, \dots, K\}} \min_{\{\mathcal{B}_k\}_{k=1}^j \in \mathcal{P}_j} \sum_{k=1}^j \sqrt{|\mathcal{B}_k|} \|\bm{x}_{\mathcal{B}_k}\|_2,
    \label{eq:lop-combinatorial}
\end{equation}
where $(\mathcal{B}_k)_{k=1}^j$ are nonoverlapping partitions of $\{1, \dots, N\}$ into $j$ contiguous blocks $\mathcal{B}_k$ as shown in the top of \cref{fig:unknown_partitions}\footnote{See \cite[Eq. (3)]{9729560} for the formal definition of $\mathcal{P}_j$.}.
The factor $\sqrt{|\mathcal{B}_k|}$ balances the penalty across blocks of different sizes, ensuring valid partition estimation in \cref{eq:lop-combinatorial}.

To avoid the combinatorial search over the family of partitions $\mathcal{P}_j$, a convex relaxation was proposed in~\cite{9729560} using a latent continuous variable $\bm{\sigma} \in \mathbb{R}^N$:
\begin{equation*}
    \Psi_\alpha (\bm{x}) := \min_{\bm{\sigma} \in \mathbb{R}^N} \sum_{n=1}^N \phi(x_n, \sigma_n) \text{ s.t. } \|\bm{D} \bm{\sigma}\|_1 \leq \alpha,
\end{equation*}
where the variational function $\phi(x, \sigma)$, mainly derived from the identity $|x| = \min_{\sigma \geq 0} (x^2/(2\sigma) + \sigma/2)$, is defined as:
\begin{equation*}
    \phi(x, \sigma) := \begin{cases} \frac{x^2}{2\sigma} + \frac{\sigma}{2} & \text{if } \sigma > 0, \\ 0 & \text{if } x = 0 \text{ and } \sigma = 0, \\ \infty & \text{otherwise,} \end{cases}
\end{equation*}
and $\alpha \geq 0$ is a Total Variation (TV) bound that controls the number of block boundaries.
\figconceptual
\tblarchcomp
By promoting piecewise-constant structures in $\bm{\sigma}$, the TV constraint $\|\bm{D} \bm{\sigma}\|_1 \leq \alpha$ identifies block boundaries as shown in \cref{fig:unknown_partitions}.

\figcompflow
The existing LOP-\ltwolone method aims to minimize:
\begin{equation}
    \minimize_{\bm{x} \in \mathbb{R}^N} f(\bm{Ax}) + \lambda \Psi_\alpha(\bm{x}),
    \label{eq:min-lop-l2l1}
\end{equation}
via a proximal splitting algorithm (\cref{alg:kuroda}), which requires $f$ to be a \textit{prox-friendly} function whose proximal operator can be computed with low complexity. For $f(\bm{u}) = \|\bm{y} - \bm{u}\|_2^2/2$, we have $\text{prox}_{\gamma f}(\bm{u}) = (\gamma \bm{y} + \bm{u}) / (\gamma + 1)$.
Here, $\lambda \geq 0$ is the regularization strength. A primary practical challenge is that hyperparameters such as $\lambda$ and $\alpha$ require careful, data-dependent tuning to ensure sufficient signal recovery as discussed in \cite[Remark 3]{9729560}.
This algorithm requires the proximal operator of $\phi$ to be computed explicitly as follows:
\begin{equation}
    \label{eq:prox_phi}
    \text{prox}_{\gamma \lambda \phi}(x, \sigma) = \begin{cases} (0, 0) & \text{if } 2 \gamma \lambda \sigma + |x|^2 \leq \gamma^2 \lambda^2, \\ (0, \sigma - \frac{\gamma \lambda}{2}) & \text{if } x = 0 \text{ and } 2 \sigma > \gamma \lambda, \\ (x - \gamma \lambda s\frac{x}{|x|}, \sigma + \gamma \lambda\frac{s^2 - 1}{2}) & \text{otherwise}, \end{cases}
\end{equation}
where $s$ is the unique positive root of the cubic equation:
\begin{equation}
    \vspace{-0.25em}
    s^3 + p s + q = 0.
    \vspace{-0.125em}
    \label{eq:cubic}
\end{equation}
The root $s$ is obtained via Cardano's formula~\cite{combettesPerspectiveFunctionsProximal2018,bauschkeRealRootsReal2023}:
\begin{equation}
    s =
    \begin{cases}
        \sqrt[3]{-\frac{q}{2} + \sqrt{-D}} + \sqrt[3]{-\frac{q}{2} - \sqrt{-D}}           & \text{if } D < 0, \\
        2\sqrt[3]{-\frac{q}{2}}                                                           & \text{if } D = 0, \\
        2 \sqrt[6]{\frac{q^2}{4} + D} \cos \left( \frac{\arctan(-2\sqrt{D}/q)}{3} \right) & \text{if } D > 0,
    \end{cases}
    \label{eq:cardano}
\end{equation}
where $p = \frac{2}{\gamma \lambda} \sigma + 1$, $q = - \frac{2}{\gamma \lambda} |x|$ and $D = -(q/2)^2 - (p/3)^3$.

\figstability
However, relying on this explicit solution in a deep unfolding framework introduces a backpropagation instability. Direct backpropagation for the computation flow (see \cref{fig:comp_flow}) requires evaluating the derivatives of Cardano's formula \cref{eq:cardano}. Near the multiple root boundary where $D \nearrow 0$, intermediate terms such as $q/\sqrt{D}$ in the chain rule diverge to infinity. This causes significant $\infty - \infty$ cancellation~\cite{goldbergWhatEveryComputer1991}, particularly in \texttt{float32} arithmetic as visualized in \cref{fig:stability_plot}.

\begin{algorithm}[t]
    \small
    \caption{Proximal Splitting for LOP-\ltwolone \cite[Alg.1]{9729560}}
    \label{alg:kuroda}
    \begin{algorithmic}
        \State \textbf{Input:} $\gamma > 0, \mu_1, \mu_2 \geq 0$, $\bm{x}^{(0)}, \bm{\sigma}^{(0)}, \bm{u}^{(0)}, \bm{\eta}^{(0)}, \bm{r}_1^{(0)}, \bm{r}_2^{(0)}$
        \For{$i = 0, \dots, L-1$}
        \State $\tilde{\bm{x}}^{(i+1)} = \bm{x}^{(i)} + \mu_1 \bm{A}^\top (\bm{r}_1^{(i)} - \mu_1 (\bm{A} \bm{x}^{(i)} - \bm{u}^{(i)}))$
        \State $\tilde{\bm{\sigma}}^{(i+1)} = \bm{\sigma}^{(i)} + \mu_2 \bm{D}^\top (\bm{r}_2^{(i)} - \mu_2 (\bm{D} \bm{\sigma}^{(i)} - \bm{\eta}^{(i)}))$
        \State $\tilde{\bm{u}}^{(i+1)} = \bm{u}^{(i)} - \mu_1 (\bm{r}_1^{(i)} - \mu_1 (\bm{A} \bm{x}^{(i)} - \bm{u}^{(i)}))$
        \State $\tilde{\bm{\eta}}^{(i+1)} = \bm{\eta}^{(i)} - \mu_2 (\bm{r}_2^{(i)} - \mu_2 (\bm{D} \bm{\sigma}^{(i)} - \bm{\eta}^{(i)}))$
        \State $(\bm{x}^{(i+1)}, \bm{\sigma}^{(i+1)}) = (\text{prox}_{\gamma \lambda \phi}(x_n, \sigma_n))_{n=1}^N$ \Comment{see \cref{eq:prox_phi}}
        \State $\bm{u}^{(i+1)} = \text{prox}_{\gamma f} (\tilde{\bm{u}}^{(i+1)})$
        \State $\bm{\eta}^{(i+1)} = \text{prox}_{\gamma\beta \|\cdot\|_1} (\tilde{\bm{\eta}}^{(i+1)})$ \Comment{when using the TV penalty $\beta\|\bm{D}\bm{\sigma}\|_1$ instead of the TV constraint $\|\bm{D}\bm{\sigma}\|_1 \leq \alpha$}
        \State $\bm{r}_1^{(i+1)} = \bm{r}_1^{(i)} - \mu_1 (\bm{A} \bm{x}^{(i+1)} - \bm{u}^{(i+1)})$
        \State $\bm{r}_2^{(i+1)} = \bm{r}_2^{(i)} - \mu_2 (\bm{D} \bm{\sigma}^{(i+1)} - \bm{\eta}^{(i+1)})$
        \EndFor
        \State \textbf{Output:} $\bm{x}^{(L)}, \bm{\sigma}^{(L)}$
    \end{algorithmic}
\end{algorithm}

\section{Proposed Method}
We propose two DU-LOP-\ltwolone architectures (\cref{tbl:arch_comp}): \textit{Proximal Splitting-based Deep Unfolded LOP-\ltwolone (PS-DU-LOP-\ltwolone)}, and \textit{Preconditioned Gradient Descent-based Deep Unfolded LOP-\ltwolone (P-GD-DU-LOP-\ltwolone)} networks to resolve the gradient instability issue.

\subsection{Method I: PS-DU-LOP-\ltwolone}
To avoid the numerical instabilities discussed in \cref{sec:lop-l2l1}, we apply implicit differentiation to the equilibrium condition $f(s, p, q) = s^3 + p s + q = 0$, directly yielding:
\begin{equation}
    \frac{\partial s}{\partial p} = - \frac{s}{3s^2 + p}, \quad \frac{\partial s}{\partial q} = - \frac{1}{3s^2 + p}.
    \label{eq:implicit-diff}
\end{equation}

\textbf{Numerical Stability}:
Unlike direct differentiation of \cref{eq:cardano}, which suffers from $\infty - \infty$ cancellation near $D=0$, these gradients address the unstable backpropagation. Although the denominator $3s^2 + p$ can vanish, the resulting singularity does not involve catastrophic cancellation, yielding more stable gradients in practice (see \cref{fig:stability_plot}).

In the unfolded architecture, we treat the per-layer parameters $(\gamma^{(k)}, \lambda^{(k)}, \mu_1^{(k)}, \mu_2^{(k)}, \beta^{(k)})$ as learnable. Note that we replace the TV constraint $\|\bm{D}\bm{\sigma}\|_1 \leq \alpha$ with the TV penalty $\beta \|\bm{D}\bm{\sigma}\|_1$ with $\beta \geq 0$ for simplicity of backpropagation.

\subsection{Method II: P-GD-DU-LOP-\ltwolone}
\begin{algorithm}[t]
    \small
    \caption{P-GD-DU-LOP-\ltwolone with WEEP-TV}
    \label{alg:pgd}
    \begin{algorithmic}
        \State \textbf{Input:} $\bm{y}, \lambda, \beta \geq 0, a > 0, b \geq 0, \eta > 0, \bm{u}^{(0)}, \bm{\sigma}^{(0)}$
        \For{$k = 0, \dots, L-1$}
        \State $\Delta_{\bm{u}} = \nabla_{\bm{u}} f(\bm{u}^{(k)} \odot \bm{\sigma}^{(k)}) + \lambda \bm{u}^{(k)}$
        \State $\Delta_{\bm{\sigma}} = \nabla_{\bm{\sigma}} f(\bm{u}^{(k)} \odot \bm{\sigma}^{(k)}) + \lambda \bm{\sigma}^{(k)} + \beta \bm{D}^\top \Omega'(\bm{D} \bm{\sigma}^{(k)})$
        \State $\bm{u}^{(k+1)} = \bm{u}^{(k)} - \eta \nabla \phi^*(\Delta_{\bm{u}})$ \Comment{Preconditioned GD}
        \State $\bm{\sigma}^{(k+1)} = \bm{\sigma}^{(k)} - \eta \nabla \phi^*(\Delta_{\bm{\sigma}})$ \Comment{Preconditioned GD}
        \EndFor
        \State \textbf{Output:} $\bm{x}^{(L)} = \bm{u}^{(L)} \odot \bm{\sigma}^{(L)}$
    \end{algorithmic}
\end{algorithm}

We propose a fully differentiable LOP-\ltwolone alternative via Deep Weight Factorization (DWF)~\cite{kolbDeepWeightFactorization2024}, reparameterizing $\bm{x} = \bm{u} \odot \bm{\sigma}$. This replaces non-differentiable penalties with smooth $\ell_2$ terms to enable purely gradient-based optimization, thereby avoiding the Cardano's formula \cref{eq:cardano}.

\newtheorem{theorem}{Theorem}
\begin{theorem}[DWF-based LOP-\ltwolone Penalty]
    The minimization problem $\min_{\bm{x}} g(\bm{x}) + \lambda \psi_K(\bm{x})$ with a differentiable function $g$ and \cref{eq:lop-combinatorial} is equivalent to the following formulation via the DWF reparameterization $\bm{x} = \bm{u} \odot \bm{\sigma} \in \mathbb{R}^N$:
    \begin{equation}
        \min_{\bm{u}, \bm{\sigma}} g(\bm{u} \odot \bm{\sigma}) + \frac{\lambda}{2} (\|\bm{u}\|_2^2 + \|\bm{\sigma}\|_2^2) \text{ s.t. } \|\bm{D} \bm{\sigma}\|_0 \leq K-1.
    \end{equation}
\end{theorem}
\begin{proof}
    The $\ell_1$ problem $\min f(\bm{x}) + \lambda \|\bm{x}\|_1$ is equivalent to its factorized form $\min_{\bm{u}, \bm{\sigma}} g(\bm{u} \odot \bm{\sigma}) + \frac{\lambda}{2} (\|\bm{u}\|_2^2 + \|\bm{\sigma}\|_2^2)$~\cite{kolbDeepWeightFactorization2024,kolbSmoothingEdgesSmooth2026}. This factorization extends to the mixed \ltwolone-norm via the following identity for a fixed block $\mathcal{B}_k$~\cite{kolbSmoothingEdgesSmooth2026,kolbDifferentiableSparsityDGating2025}:
    \begin{align*}
         & \min_{\bm{x} \in \mathbb{R}^N} \left( g(\bm{x}) + \lambda \sum_{k=1}^j \sqrt{|\mathcal{B}_k|} \|\bm{x}_{\mathcal{B}_k}\|_2 \right)                                        \\
         & = \min_{\bm{u}, \bm{\tau}} \left( g(\bm{u} \odot \bm{\tau}^\dagger) + \frac{\lambda}{2} \sum_{k=1}^j (\|\bm{u}_{\mathcal{B}_k}\|_2^2 + |\mathcal{B}_k| \tau_k^2) \right),
    \end{align*}
    which yields a differentiable form of the mixed \ltwolone-norm. This equivalence holds for any contiguous partition $\{\mathcal{B}_k\}_{k=1}^j \in \mathcal{P}_j$ and $j \in \{1, ..., K-1\}$. By minimizing over all partitions, we obtain the DWF-based LOP-\ltwolone penalty:
    \begin{align*}
         & \min_{\bm{x} \in \mathbb{R}^N} \min_{j, \{\mathcal{B}_k\}_{k=1}^j} \left( g(\bm{x}) + \lambda \sum_{k=1}^j \sqrt{|\mathcal{B}_k|} \|\bm{x}_{\mathcal{B}_k}\|_2 \right)                                       \\
         & = \min_{j, \{\mathcal{B}_k\}_{k=1}^j} \min_{\bm{u}, \bm{\tau}} \left( g(\bm{u} \odot \bm{\tau}^\dagger) + \frac{\lambda}{2} \sum_{k=1}^j (\|\bm{u}_{\mathcal{B}_k}\|_2^2 + |\mathcal{B}_k| \tau_k^2) \right) \\
         & = \min_{\bm{u}, \bm{\sigma}} g(\bm{u} \odot \bm{\sigma}) + \frac{\lambda}{2} (\|\bm{u}\|_2^2 + \|\bm{\sigma}\|_2^2) \text{ s.t. } \|\bm{D} \bm{\sigma}\|_0 \leq K-1,
    \end{align*}
    where $\bm{\tau} \in \mathbb{R}^j$ contains blockwise factors and $\bm{\sigma} = \bm{\tau}^\dagger \in \mathbb{R}^N$ is its piecewise-constant expansion, i.e., $\sigma_n = \tau_k$ for $n \in \mathcal{B}_k$~\cite{kolbSmoothingEdgesSmooth2026}. The constraint $\|\bm{D} \bm{\sigma}\|_0 \leq K-1$ encodes that $\bm{\sigma}$ has at most $K$ constant pieces, where each block boundary corresponds to a non-zero jump in the difference $\bm{D}\bm{\sigma}$.
\end{proof}

We can relax the discrete $\ell_0$ constraint with the WEEP penalty $\Omega$~\cite{furuhashiWEEPDifferentiableNonconvex2026} (see \cref{sec:experiments}). The overall objective is:
\begin{equation*}
    \minimize_{\bm{u}, \bm{\sigma} \in \mathbb{R}^N} g(\bm{u} \odot \bm{\sigma}) + \frac{\lambda}{2} (\|\bm{u}\|_2^2 + \|\bm{\sigma}\|_2^2) + \beta \cdot \Omega(\bm{D} \bm{\sigma}).
\end{equation*}

\textbf{Numerical Stability}:
This gradient-based approach eliminates the need for proximal operators and the associated cubic equation solving required in the LOP-\ltwolone framework (\cref{sec:lop-l2l1}), avoiding its potential singularities.

\textbf{Flexibility}: The approach facilitates the use of nonconvex smooth losses for $g$, such as $g(\bm{x}) = \sum_{i} h((\bm{y} - \bm{A}\bm{x})_i; a, b)$ where WEEP's influence saturates for large residuals more strongly than the convex loss like the $\ell_2$ loss and Huber loss~\cite{huberRobustEstimationLocation1964}. This property more effectively ignores impulsive outliers and ensures more robust recovery.

While the DWF loss landscape is non-Lipschitz smooth and can lead to instability under vanilla gradient descent, we achieve more stable empirical convergence by using Preconditioned Gradient Descent (P-GD, \cref{alg:pgd})~\cite{bodardEscapingSaddlePoints2025}, where $\nabla \phi^*$ is the gradient of the convex conjugate $\phi^*(\bm{z}) = \sup_{\bm{x}} \langle \bm{x}, \bm{z} \rangle - \phi(\bm{x})$ of a reference function $\phi$ (e.g., $\cosh(\|\bm{z}\|) - 1$). That provides implicit gradient clipping, which also tends to accelerate training in nonconvex optimization~\cite{zhangWhyGradientClipping2020,zhangImprovedAnalysisClipping2020}.

\figphasetransition

\section{Numerical Experiments}
\label{sec:experiments}
\begin{table}[t]
    \caption{Model complexity comparison when using $\ell_2$ loss.}
    \label{tbl:complexity}
    \centering\footnotesize
    \setlength{\tabcolsep}{3pt}
    \begin{tabular}{lcl} \hline
        Model                        & Params/Layer & Learnable Variables                    \\ \hline
        LT-ISTA \& DWF ($\beta = 0$) & 2            & $\eta, \lambda$                        \\
        P-GD (L1-TV)                 & 3            & $\eta, \lambda, \beta$                 \\
        \textbf{P-GD (WEEP-TV)}      & 5            & $\eta, \lambda, \beta, a, b$           \\
        \textbf{PS (L1-TV)}          & 5            & $\gamma, \lambda, \beta, \mu_1, \mu_2$ \\ \hline
    \end{tabular}
\end{table}
\figptoutlier

We evaluate DU-LOP-\ltwolone through synthetic block-sparse signal recovery tasks, benchmarking against Learned Threshold ISTA (LT-ISTA) and a DWF baseline without TV penalty ($\beta = 0$), both of which do not explicitly consider underlying block structures. LT-ISTA is a variant of Learned ISTA (LISTA)~\cite{gregorLearningFastApproximations2010} that only learns the layer-specific thresholds and step sizes ($\eta^{(k)}, \lambda^{(k)}$).

\textbf{WEEP Penalty}:
We adopt the WEEP function~\cite{furuhashiWEEPDifferentiableNonconvex2026} as a differentiable nonconvex surrogate for the $\ell_0$ penalty.
We utilize it both as a TV penalty (WEEP-TV) for the vector $\bm{\sigma} \in \mathbb{R}^N$ and as a robust data fidelity loss $f$ due to its influence saturation against outliers.
WEEP-TV saturates for large gradients, enabling sharper block boundary discovery and more accurate partitioning than $\ell_1$-based penalties.

\textbf{Complexity}:
\cref{tbl:complexity} shows parameter efficiency. Our WEEP-TV variant requires 75 parameters for 15 layers, a negligible overhead justified by substantial recovery gains.

\textbf{Experimental Setup}:
We generated block-sparse signals $\bm{x} \in \mathbb{R}^N$ ($N=64$) by dividing the vector into $B=16$ blocks of size $B_{size}=4$. For a given sparsity $k$, we randomly activated $K = k/B_{size}$ blocks with coefficients drawn from the uniform distribution $\mathcal{U}([0.5, 1.0])$ and random signs. Inactive blocks were set to zero. Measurements $\bm{y} = \bm{A} \bm{x}$ were obtained using a Gaussian random matrix $\bm{A} \in \mathbb{R}^{M \times N}$ normalized such that $A_{i,j} \sim \mathcal{N}(0, M^{-1})$.
We evaluate signal recovery using three types of data fidelities $f$ in \cref{eq:min-lop-l2l1}: (i) the $\ell_2$ loss, (ii) Huber loss, and (iii) WEEP loss. For (ii) and (iii), the parameters are learned for each layer.

For each grid point, unfolded models (depth $L=15$) were independently trained end-to-end on 50 signal pairs using the Adam optimizer (learning rate $10^{-3}$) to minimize the mean squared error (MSE) $\|\hat{\bm{x}} - \bm{x}\|_2^2$, and evaluated on 50 test signals. Recovery success is defined by the signal-to-noise ratio (SNR): $10 \log_{10} (\|\bm{x}\|_2^2 / \|\hat{\bm{x}} - \bm{x}\|_2^2) \, [\mathrm{dB}]$. Trials with $\mathrm{SNR}$ exceeding 20\,dB were classified as successful.

\textbf{Results}:
We swept measurement $M/N$ and sparsity $K/N$ ratios across 50 trials per grid point.
\cref{fig:phase_transition} shows that the proposed PS-DU-LOP-\ltwolone yields the highest success rates over all. While P-GD-DU-LOP-\ltwolone (WEEP-TV) exhibit slightly lower recovery rates than the PS alternative, it still outperforms the $\ell_1$ baselines.
Method labels in \cref{fig:phase_transition,fig:pt_outlier} follow the format "Method (Regularizer, Fidelity)"; if the fidelity is omitted, the $\ell_2$ loss is applied.

For outlier robustness under additive Bernoulli impulsive noise (contamination probability $0.05$, magnitude $5.0$), methods using $\ell_2$ fidelity fail while our WEEP variants achieve the highest success (see \cref{fig:pt_outlier}). It emphasizes the robustness of the WEEP penalty against outliers, which is an advantage of our proposed gradient-based method.

\section{Conclusion}
We proposed DU-LOP-\ltwolone, a deep unfolding framework for block-sparse recovery under unknown partitions that automates hyperparameter tuning. Two methods were presented: a proximal splitting-based method and a gradient-based variant using smooth reparameterization (DWF). The latter enables nonconvex data fidelity for robust recovery against outliers. Experiments confirmed competitive recovery performance and high outlier resilience.


\begin{spacing}{1.0}
    \bibliographystyle{plain}
    \bibliography{refs}

\begin{thebibliography}{10}

\bibitem{balatsoukas-stimmingDeepUnfoldingCommunications2019}
Alexios {Balatsoukas-Stimming} and Christoph Studer.
\newblock Deep {{Unfolding}} for {{Communications Systems}}: {{A Survey}} and {{Some New Directions}}.
\newblock In {\em IEEE International Workshop on Signal Processing Systems (SiPS)}, pages 266--271, 2019.

\bibitem{baraniukModelBasedCompressiveSensing2010}
Richard~G. Baraniuk, Volkan Cevher, Marco~F. Duarte, and Chinmay Hegde.
\newblock Model-{{Based Compressive Sensing}}.
\newblock {\em IEEE Transactions on Information Theory}, 56(4):1982--2001, 2010.

\bibitem{bauschkeRealRootsReal2023}
Heinz~H. Bauschke, Manish~Krishan Lal, and Xianfu Wang.
\newblock Real roots of real cubics and optimization.
\newblock arXiv:2302.10731, 2023.

\bibitem{bodardEscapingSaddlePoints2025}
Alexander Bodard and Panagiotis Patrinos.
\newblock Escaping saddle points without {{Lipschitz}} smoothness: The power of nonlinear preconditioning.
\newblock In {\em Annual Conference on Neural Information Processing Systems (NeurIPS)}, 2025.

\bibitem{borgerdingAMPInspiredDeepNetworks2017}
Mark Borgerding, Philip Schniter, and Sundeep Rangan.
\newblock {{AMP-Inspired Deep Networks}} for {{Sparse Linear Inverse Problems}}.
\newblock {\em IEEE Transactions on Signal Processing}, 65(16):4293--4308, 2017.

\bibitem{carrilloRobustSamplingReconstruction2010}
Rafael~E. Carrillo, Kenneth~E. Barner, and Tuncer~C. Aysal.
\newblock Robust {{Sampling}} and {{Reconstruction Methods}} for {{Sparse Signals}} in the {{Presence}} of {{Impulsive Noise}}.
\newblock {\em IEEE Journal of Selected Topics in Signal Processing}, 4(2):392--408, 2010.

\bibitem{combettesPerspectiveFunctionsProximal2018}
Patrick~L. Combettes and Christian~L. M{\"u}ller.
\newblock Perspective functions: {{Proximal}} calculus and applications in high-dimensional statistics.
\newblock {\em Journal of Mathematical Analysis and Applications}, 457(2):1283--1306, 2018.

\bibitem{eldarBlocksparseSignalsUncertainty2010}
Yonina~C. Eldar, Patrick Kuppinger, and Helmut Bolcskei.
\newblock Block-sparse signals: {{Uncertainty}} relations and efficient recovery.
\newblock {\em IEEE Transactions on Signal Processing}, 58(6):3042--3054, 2010.

\bibitem{fangPatterncoupledSparseBayesian2015}
Jun Fang, Yanning Shen, Hongbin Li, and Pu~Wang.
\newblock Pattern-coupled sparse bayesian learning for recovery of block-sparse signals.
\newblock {\em IEEE Transactions on Signal Processing}, 63(2):360--372, 2015.

\bibitem{fuDeepUnfoldingNetwork2021}
Rong Fu, Vincent Monardo, Tianyao Huang, and Yimin Liu.
\newblock Deep {{Unfolding Network}} for {{Block-Sparse Signal Recovery}}.
\newblock In {\em IEEE International Conference on Acoustics, Speech and Signal Processing (ICASSP)}, pages 2880--2884, 2021.

\bibitem{furuhashiWEEPDifferentiableNonconvex2026}
Takanobu Furuhashi, Hidekata Hontani, Qibin Zhao, and Tatsuya Yokota.
\newblock {{WEEP}}: {{A Differentiable Nonconvex Sparse Regularizer}} via {{Weakly-Convex Envelope}}.
\newblock In {\em IEEE International Conference on Acoustics, Speech and Signal Processing (ICASSP)}, 2026.

\bibitem{gaoDeepLearningBasedChannel2023}
Jiabao Gao, Caijun Zhong, Geoffrey~Ye Li, Joseph~B. Soriaga, and Arash Behboodi.
\newblock Deep {{Learning-Based Channel Estimation}} for {{Wideband Hybrid MmWave Massive MIMO}}.
\newblock {\em IEEE Transactions on Communications}, 71(6):3679--3693, 2023.

\bibitem{gaoBlockSparseRPCASalient2014}
Zhi Gao, Loong-Fah Cheong, and Yu-Xiang Wang.
\newblock Block-{{Sparse RPCA}} for {{Salient Motion Detection}}.
\newblock {\em IEEE Transactions on Pattern Analysis and Machine Intelligence}, 36(10):1975--1987, 2014.

\bibitem{goldbergWhatEveryComputer1991}
David Goldberg.
\newblock What every computer scientist should know about floating-point arithmetic.
\newblock {\em ACM Comput. Surv.}, 23(1):5--48, 1991.

\bibitem{gregorLearningFastApproximations2010}
Karol Gregor and Yann LeCun.
\newblock Learning fast approximations of sparse coding.
\newblock In {\em International Conference on Machine Learning (ICML)}, pages 399--406, 2010.

\bibitem{hersheyDeepUnfoldingModelBased2014}
John~R. Hershey, Jonathan~Le Roux, and Felix Weninger.
\newblock Deep {{Unfolding}}: {{Model-Based Inspiration}} of {{Novel Deep Architectures}}.
\newblock arXiv:1409.2574, 2014.

\bibitem{huberRobustEstimationLocation1964}
Peter~J. Huber.
\newblock Robust {{Estimation}} of a {{Location Parameter}}.
\newblock {\em The Annals of Mathematical Statistics}, 35(1):73--101, 1964.

\bibitem{itoTrainableISTASparse2019}
Daisuke Ito, Satoshi Takabe, and Tadashi Wadayama.
\newblock Trainable {{ISTA}} for {{Sparse Signal Recovery}}.
\newblock {\em IEEE Transactions on Signal Processing}, 67(12):3113--3125, 2019.

\bibitem{kolbDifferentiableSparsityDGating2025}
Chris Kolb, Laetitia Frost, Bernd Bischl, and David R{\"u}gamer.
\newblock Differentiable {{Sparsity}} via {{D-Gating}}: {{Simple}} and {{Versatile Structured Penalization}}.
\newblock In {\em Annual Conference on Neural Information Processing Systems (NeurIPS)}, 2025.

\bibitem{kolbSmoothingEdgesSmooth2026}
Chris Kolb, Christian~L. M{\"u}ller, Bernd Bischl, and David R{\"u}gamer.
\newblock Smoothing the {{Edges}}: {{Smooth Optimization}} for {{Sparse Regularization Using Hadamard Overparametrization}}.
\newblock {\em Machine Learning}, 115(4):87, 2026.

\bibitem{kolbDeepWeightFactorization2024}
Chris Kolb, Tobias Weber, Bernd Bischl, and David R{\"u}gamer.
\newblock Deep {{Weight Factorization}}: {{Sparse Learning Through}} the {{Lens}} of {{Artificial Symmetries}}.
\newblock In {\em International Conference on Learning Representations (ICLR)}, 2024.

\bibitem{kurodaConvexnonconvexFrameworkEnhancing2025}
Hiroki Kuroda.
\newblock A convex-nonconvex framework for enhancing minimization induced penalties.
\newblock {\em Journal of the Franklin Institute}, 362(15):107969, 2025.

\bibitem{kurodaTheoreticalValidationLatent2025}
Hiroki Kuroda, Renato Luis~Garrido Cavalcante, and Masahiro Yukawa.
\newblock Theoretical {{Validation}} of the {{Latent Optimally Partitioned-L2}}/{{L1 Penalty}} with {{Application}} to {{Angular Power Spectrum Estimation}}, 2025.

\bibitem{9729560}
Hiroki Kuroda and Daichi Kitahara.
\newblock Block-sparse recovery with optimal block partition.
\newblock {\em IEEE Transactions on Signal Processing}, 70:1506--1520, 2022.

\bibitem{liStructuredSparseCoding2022}
Zhenni Li, Yujie Li, Benying Tan, Shuxue Ding, and Shengli Xie.
\newblock Structured {{Sparse Coding With}} the {{Group Log-regularizer}} for {{Key Frame Extraction}}.
\newblock {\em IEEE/CAA Journal of Automatica Sinica}, 9(10):1818--1830, 2022.

\bibitem{liuBackgroundSubtractionBased2015}
Xin Liu, Guoying Zhao, Jiawen Yao, and Chun Qi.
\newblock Background {{Subtraction Based}} on {{Low-Rank}} and {{Structured Sparse Decomposition}}.
\newblock {\em IEEE Transactions on Image Processing}, 24(8):2502--2514, 2015.

\bibitem{lvGroupLassoStable2011}
Xiaolei Lv, Guoan Bi, and Chunru Wan.
\newblock The {{Group Lasso}} for {{Stable Recovery}} of {{Block-Sparse Signal Representations}}.
\newblock {\em IEEE Transactions on Signal Processing}, 59(4):1371--1382, 2011.

\bibitem{mongaAlgorithmUnrollingInterpretable2021}
Vishal Monga, Yuelong Li, and Yonina~C. Eldar.
\newblock Algorithm {{Unrolling}}: {{Interpretable}}, {{Efficient Deep Learning}} for {{Signal}} and {{Image Processing}}.
\newblock {\em IEEE Signal Processing Magazine}, 38(2):18--44, 2021.

\bibitem{qianHyperspectralImageryRestoration2013}
Yuntao Qian and Minchao Ye.
\newblock Hyperspectral {{Imagery Restoration Using Nonlocal Spectral-Spatial Structured Sparse Representation With Noise Estimation}}.
\newblock {\em IEEE Journal of Selected Topics in Applied Earth Observations and Remote Sensing}, 6(2):499--515, 2013.

\bibitem{santBlockSparseSignalRecovery2022}
Aditya Sant, Markus Leinonen, and Bhaskar~D. Rao.
\newblock Block-{{Sparse Signal Recovery}} via {{General Total Variation Regularized Sparse Bayesian Learning}}.
\newblock {\em IEEE Transactions on Signal Processing}, 70:1056--1071, 2022.

\bibitem{shlezingerDeepUnfoldingRecent2025}
Nir Shlezinger, Santiago Segarra, Yi~Zhang, Dvir Avrahami, Zohar Davidov, Tirza Routtenberg, and Yonina~C. Eldar.
\newblock Deep {{Unfolding}}: {{Recent Developments}}, {{Theory}}, and {{Design Guidelines}}.
\newblock arXiv:2512.03768, 2025.

\bibitem{stojnicReconstructionBlockSparseSignals2009}
Mihailo Stojnic, Farzad Parvaresh, and Babak Hassibi.
\newblock On the {{Reconstruction}} of {{Block-Sparse Signals With}} an {{Optimal Number}} of {{Measurements}}.
\newblock {\em IEEE Transactions on Signal Processing}, 57(8):3075--3085, 2009.

\bibitem{studerRecoverySparselyCorrupted2012}
Christoph Studer, Patrick Kuppinger, Graeme Pope, and Helmut Bolcskei.
\newblock Recovery of sparsely corrupted signals.
\newblock {\em IEEE Transactions on Information Theory}, 58(5):3115--3130, 2012.

\bibitem{wanRobustBayesianCompressed2017}
Qian Wan, Huiping Duan, Jun Fang, Hongbin Li, and Zhengli Xing.
\newblock Robust {{Bayesian}} compressed sensing with outliers.
\newblock {\em Signal Processing}, 140:104--109, 2017.

\bibitem{wangEnhancedISARImaging2014}
Lu~Wang, Lifan Zhao, Guoan Bi, Chunru Wan, and Lei Yang.
\newblock Enhanced {{ISAR Imaging}} by {{Exploiting}} the {{Continuity}} of the {{Target Scene}}.
\newblock {\em IEEE Transactions on Geoscience and Remote Sensing}, 52(9):5736--5750, 2014.

\bibitem{wuRPCANetDeepUnfolding2024}
Fengyi Wu, Tianfang Zhang, Lei Li, Yian Huang, and Zhenming Peng.
\newblock {{RPCANet}}: {{Deep Unfolding RPCA Based Infrared Small Target Detection}}.
\newblock In {\em IEEE/CVF Winter Conference on Applications of Computer Vision (WACV)}, pages 4797--4806, 2024.

\bibitem{yuanL0TVNewMethod2015}
Ganzhao Yuan and Bernard Ghanem.
\newblock {{$\ell$0TV}}: {{A}} new method for image restoration in the presence of impulse noise.
\newblock In {\em IEEE/CVF Conference on Computer Vision and Pattern Recognition (CVPR)}, pages 5369--5377, 2015.

\bibitem{yuanModelSelectionEstimation2006}
Ming Yuan and Yi~Lin.
\newblock Model selection and estimation in regression with grouped variables.
\newblock {\em Journal of the Royal Statistical Society: Series B (Statistical Methodology)}, 68(1):49--67, 2006.

\bibitem{zhangImprovedAnalysisClipping2020}
Bohang Zhang, Jikai Jin, Cong Fang, and Liwei Wang.
\newblock Improved analysis of clipping algorithms for non-convex optimization.
\newblock In {\em Annual Conference on Neural Information Processing Systems (NeurIPS)}, pages 15511--15521, 2020.

\bibitem{zhangISTANetInterpretableOptimizationInspired2018}
Jian Zhang and Bernard Ghanem.
\newblock {{ISTA-Net}}: {{Interpretable Optimization-Inspired Deep Network}} for {{Image Compressive Sensing}}.
\newblock In {\em IEEE/CVF Conference on Computer Vision and Pattern Recognition (CVPR)}, pages 1828--1837, 2018.

\bibitem{zhangWhyGradientClipping2020}
Jingzhao Zhang, Tianxing He, Suvrit Sra, and Ali Jadbabaie.
\newblock Why gradient clipping accelerates training: {{A}} theoretical justification for adaptivity.
\newblock arXiv:1905.13655, 2020.

\bibitem{zhangExtensionSBLAlgorithms2013}
Zhilin Zhang and Bhaskar~D. Rao.
\newblock Extension of {{SBL Algorithms}} for the {{Recovery}} of {{Block Sparse Signals With Intra-Block Correlation}}.
\newblock {\em IEEE Transactions on Signal Processing}, 61(8):2009--2015, 2013.

\bibitem{zhaoHyperspectralAnomalyDetection2023}
Yin-Ping Zhao, Hongyan Li, Yongyong Chen, Zhen Wang, and Xuelong Li.
\newblock Hyperspectral {{Anomaly Detection}} via {{Structured Sparsity Plus Enhanced Low-Rankness}}.
\newblock {\em IEEE Transactions on Geoscience and Remote Sensing}, 61:1--15, 2023.

\end{thebibliography}
\end{spacing}

\end{document}